\documentclass[sigconf]{acmart}

\usepackage[linesnumbered,ruled,vlined]{algorithm2e}
\usepackage{epstopdf}
\usepackage{subfigure}
\usepackage{enumitem}
\usepackage{multirow}
\usepackage{url}
\usepackage{colortbl}
\SetKwInput{KwInput}{Input}
\SetKwInput{KwOutput}{Output}

\def\lerand{Lemane}

\newcommand{\vect}[1]{\boldsymbol{#1}}
\def\header{\vspace{2mm} \noindent}
\def\prox{S}
\definecolor{LightSteelBlue}{RGB}{213,229,255}

\newtheorem{theorem} {Theorem}
\newtheorem{lemma} {Lemma}

\AtBeginDocument{
  \providecommand\BibTeX{{
    \normalfont B\kern-0.5em{\scshape i\kern-0.25em b}\kern-0.8em\TeX}}}

\copyrightyear{2021}
\acmYear{2021}
\setcopyright{acmlicensed}\acmConference[KDD '21]{Proceedings of the 27th ACM SIGKDD Conference on Knowledge Discovery and Data Mining}{August 14--18, 2021}{Virtual Event, Singapore}
\acmBooktitle{Proceedings of the 27th ACM SIGKDD Conference on Knowledge Discovery and Data Mining (KDD '21), August 14--18, 2021, Virtual Event, Singapore}
\acmPrice{15.00}
\acmDOI{10.1145/3447548.3467296}
\acmISBN{978-1-4503-8332-5/21/08}

\begin{document}
\fancyhead{}

\title{Learning Based Proximity Matrix Factorization for Node Embedding}

\author{Xingyi Zhang}
\author{Kun Xie}
\author{Sibo Wang}
\authornote{Sibo Wang is the corresponding author.}
\affiliation{
  \institution{The Chinese University of Hong Kong}
  \city{Hong Kong SAR}
    \country{China}
}
\email{xyzhang, xiekun, swang@se.cuhk.edu.hk}

\author{Zengfeng Huang}
\affiliation{
  \institution{Fudan University}
  \city{Shanghai}
  \country{China}
}
\email{huangzf@fudan.edu.cn}

\begin{abstract}
\label{sec:sec-abstract}
Node embedding learns a low-dimensional representation for each node in the graph. Recent progress on node embedding shows that proximity matrix factorization methods gain superb performance and scale to large graphs with millions of nodes. Existing approaches first define a proximity matrix and then learn the embeddings that fit the proximity by matrix factorization. Most existing matrix factorization methods adopt the same proximity for different tasks, while it is observed that different tasks and datasets may require different proximity, limiting their representation power. 

Motivated by this, we propose {\em Lemane}, a framework with trainable proximity measures, which can be learned to best suit the datasets and tasks at hand automatically. Our method is end-to-end, which incorporates differentiable SVD in the pipeline so that the parameters can be trained via backpropagation. However, this learning process is still expensive on large graphs. To improve the scalability, we train proximity measures only on carefully subsampled graphs, and then apply standard proximity matrix factorization on the original graph using the learned proximity. Note that, computing the learned proximities for each pair is still expensive for large graphs, and existing techniques for computing proximities are not applicable to the learned proximities. Thus, we present generalized push techniques to make our solution scalable to large graphs with millions of nodes. Extensive experiments show that our proposed solution outperforms existing solutions on both link prediction and node classification tasks on almost all datasets.

\end{abstract}



\maketitle

\section{Introduction}
\label{sec:sec-Introduction}

Node embedding is the task to map nodes in the original graph into low-dimensional representations. For each node, it outputs an embedding vector and the embedding vectors play an important role in preserving not only the structural information but also other underlying properties in the graph. 
These vectors can be fed into machine learning models, facilitating widespread machine learning tasks, such as node classification \cite{deepwalk,node2vec,LINE}, link prediction \cite{GraphGAN,verse}, graph reconstruction \cite{STRAP,nrp}, and recommendation \cite{recommend1,recommend2}.

An important category of node embedding methods with superb performance is the ones using proximity matrix factorization.
For such solutions, they first define the proximity matrix $\vect{\prox}$ for the nodes in the input graph where $\vect{\prox}(i, j)$ is the proximity measure of node $j$ with respect to node $i$. Different methods may adopt different proximity measures and {\em personalized PageRank (PPR)} is a popular choice of proximity measure in node embedding. For example, PPR is adopted in NRP \cite{nrp} as the proximity. In STRAP \cite{STRAP}, the authors further propose to adopt $\vect{\pi}_u(v)+\vect{\pi}^T_v(u)$ as the proximity where $\vect{\pi}_u(v)$ is the PPR of $v$ with respect to $u$ and $\vect{\pi}^T_v(u)$ is the PPR of $u$ with respect to $v$ on the transpose graph $G^T$ by reversing the direction of each edge in $G$. Given the proximity matrix $\vect{\prox}$, two embedding vectors $\vect{x}_u$ and $\vect{y}_u$ are derived such that $\vect{x}_u\cdot \vect{y}_v \sim \vect{\prox}(u,v)$. For existing solutions in this category, the embedding vectors are typically obtained by singular value decomposition (SVD) or eigen-decomposition on $\vect{\prox}$ or on a sparse matrix closely related to $\vect{\prox}$.

Despite their success, all existing matrix factorization approaches aim to learn an embedding that preserves the chosen proximity without considering if the proximity is suitable for the task on the dataset or not. However, it is observed that different tasks and datasets may require different proximities to achieve high performance, limiting the representation power of such solutions. In addition, it is shown in existing node embedding methods, e.g., STRAP \cite{STRAP} and NetMF \cite{NetMF}, that non-linear operations (such as taking logarithm or softmax) on the proximity matrix can help improve the representation power of the embedding. Nevertheless, most latest matrix factorization methods, like HOPE \cite{HOPE}, AROPE \cite{AROPE}, and NRP \cite{nrp},  do not explicitly derive the proximity matrix, which limits their representation powers. 
For those methods that explicitly derive the proximity matrix, it indicates that the matrix factorization needs to be taken on the final proximity matrix. Thus, to derive trainable proximity measures, the model needs to be end-to-end and includes the proximity computation as well as the SVD decomposition into the training process. This further imposes challenges especially when the input graph has millions of nodes.

\header
{\bf Contribution.} Motivated by the limitation of existing solutions, we present an effective framework {\em \lerand}\footnote{\underline{Le}arning based Proximity \underline{Ma}trix Factorization for \underline{N}ode \underline{E}mbedding} with trainable proximity measures, which can be learned to best suit the datasets and tasks at hand automatically. Our trainable proximity measure is inspired by personalized PageRank (PPR) \cite{pagerank}. The PPR $\vect{\pi}_u(v)$ can be defined as the probability that an $\alpha$-discounted random walk from $u$ stops at node $v$, where an $\alpha$-discounted random walk from a source node $u$ has $\alpha$ probability to stop at the current node and has $(1-\alpha)$ probability to randomly jump to an out-neighbor of the current node. For $\alpha$-discounted random walks, the majority will always be the one-hop random walks, which may not be the most representative one for the task at hand. This motivates us to learn a more representative random walk for our trainable proximity measure. Instead of fixing the stopping probability as $\alpha$ at each step, our trainable proximity is defined on a supervised random walk where the stopping probability at the $l$-th hop is learned by our defined loss function. Then, our trainable proximity of node $v$ with respect to $u$, dubbed as the supervised PPR $\vect{\prox}(u,v)$, is the probability that the supervised random walk from $u$ stops at $v$.

To learn the stopping probabilities at each hop for the supervised random walk, we design different loss functions for different tasks in order to learn a more representative proximity for the task at hand. In this paper, we focus on two popular tasks of node embedding: link prediction and node classification. Given the loss functions and trainable parameters, we then design an end-to-end method which incorporates a differentiable SVD in the pipeline so that the parameters can be trained via backpropagation. Our solution is mainly inspired by previous work on learning-based low-rank approximations \cite{low-rank}, which includes a differentiable SVD that allows the gradients to flow easily in the framework to solve the low-rank approximation problem. In our framework, with the differentiable SVD, the gradients can easily flow from the loss function and differentiable SVD to our proximity matrix computation, which is determined by the training parameters, i.e., the stopping probabilities at each hop of the supervised random walk. 

However, the above training process is too expensive and does not scale to large graphs. To improve the scalability, we train the stopping probabilities for the supervised random walk only on carefully subsampled graphs, and then apply standard proximity matrix factorization on the original graph using the learned proximity. Our main observation is that the stopping probabilities at each hop of the supervised random walk are node-independent. Thus, with a carefully subsampled graph, the learned stopping probability at each hop for the task should still be similar to that on the input graph. Motivated by this, we present an effective subgraph sampling based method to train the parameters on multiple sampled subgraphs, which improves the scalability of our {\lerand}.

Finally, given learned probabilities, computing learned proximities for each node-pair is still expensive for large graphs, and existing efficient algorithms for computing proximities, like PPR, are not applicable to the learned proximities. Thus, we present generalized push techniques to make our solution scalable to large graphs with millions of nodes. 
Given a source node $s$, our generalized push algorithm computes an approximate supervised PPR score for each node with respect to $s$ with $O(\frac{1}{\delta})$ cost where $\delta$ is a parameter to control the quality of approximate supervised PPR scores. Then, the proximity matrix can be computed with $O(\frac{n}{\delta})$ cost and takes $O(\frac{n}{\delta})$ space. 
A sparse SVD algorithm is applied on the proximity matrix $\vect{\prox}$ to derive the final embedding with $O(\frac{n}{\delta}+n\cdot d^2)$ running cost, where $d$ is the embedding dimension. 

In our experiment, we compare our {\lerand} against 15 existing node embedding methods on the link prediction and node classification tasks using 6 real datasets with up to 3 million nodes and 117 million edges. Extensive experiments show that our {\lerand} outperforms existing methods on both link prediction and node classification tasks on almost all datasets.

\section{Related Work}
\label{sec:sec-Preliminary}

There are three basic categories of node embedding methods:  skip-gram methods, matrix factorization methods, and neural network methods. Next, We briefly review existing works for each category.

\header
\textbf{Skip-gram methods.} The methods in this category are inspired by the great success of the word2vec model \cite{word2vec} for natural language processing. DeepWalk \cite{deepwalk} first proposes to train embedding vectors by feeding truncated random walks to the Skip-gram model. The nodes sampled from the random walks are then treated as the positive samples. Subsequent methods try to explore more representative random walks to feed into the Skip-gram model. LINE \cite{LINE} adopts one-hop and two-hop random walks while Node2vec \cite{node2vec} proposes to explore higher-order random walks that exploits both DFS and BFS nature of the graph. VERSE \cite{verse} and APP \cite{APP} adopt $\alpha$-discounted random walks to obtain positive samples. Recently, InfiniteWalk \cite{infwalk} studies DeepWalk in the limit as the window size goes to infinite, linking DeepWalk to graph Laplacian matrix.

\header
\textbf{Matrix factorization methods.} Another idea in node embedding is to do matrix factorization on a chosen proximity matrix. To explicitly derive the proximity matrix, e.g., the case in NetMF \cite{NetMF}, it typically takes $\Theta(n^2)$ cost and is too expensive for large graphs.
To avoid the $\Theta(n^2)$ running cost, HOPE \cite{HOPE}, AROPE \cite{AROPE}, and NRP \cite{nrp} are proposed to derive the embedding without explicitly computing the proximity matrix. 
For instance, instead of computing all-pair proximity scores and then decomposing the proximity matrix $\vect{\prox}$, NRP turns to do SVD on the adjacency matrix, which is sparse for most real-life graphs, reducing the embedding computational cost. 
Another solution to avoid the $\Theta(n^2)$ cost is to calculate a sparsified proximity matrix $\vect{\prox}$. The representative is STRAP \cite{STRAP}, which imposes a threshold $\delta$ and returns at most $O(\frac{1}{\delta})$ proximity scores no smaller than $\delta$ for each node, making the proximity matrix of $O(\frac{n}{\delta})$ size. An SVD is then applied to the sparsified proximity matrix. Since the second solution explicitly derives the proximity matrix, it allows to take non-linear operations on the proximity matrix, improving the representation powers.

\header
\textbf{Neural network methods.} 
Deep learning provides an alternative solution to generate node embeddings. SDNE \cite{SDNE} and DNGR \cite{DNGR} employ multi-layer auto-encoders with a target matrix to generate embeddings. DRNE \cite{DRNE} utilizes the layer normalized LSTM \cite{lstm} to generate node embeddings by aggregating the representations of neighbors of each node recursively. GraphGAN \cite{GraphGAN} adopts the well-known generative adversarial networks \cite{gan} into graph representation learning via an adversarial minimax game. AW \cite{AW} proposes a novel attention model on the power series of the transition matrix, which guides the random walk to pay attention to important positions within the random walks by optimizing an upstream objective. The bottleneck of these solutions is the high computational cost, which restricts these methods to small graphs.

\noindent
\textbf{Other methods.}
There are also several methods that do not belong to the above three categories. For example, GraphWave \cite{graphwave} learns node representations by leveraging heat wavelet diffusion patterns in an unsupervised way. NetHiex \cite{nethiex} captures the underlying hierarchical taxonomy of the graph to learn node representations with multiple components. RaRE \cite{rare} proposes a node embedding method that considers both social rank and proximity of nodes, and separately learns two representations for a node. AutoNE \cite{AutoNE} incorporates AutoML into node embedding, which can automatically optimize the hyperparameters from the subgraphs of the original graph. PBG \cite{pbg} presents a distributed embedding system that uses the block decomposition of the adjacency matrix as a partition method to scale to arbitrary large graphs. A recent work, GraphZOOM \cite{graphzoom}, first generates subgraphs based on a fused graph and then applies existing approaches to generate node embeddings. Since these methods do not preserve any node pair proximity, the main concern is their insufficient effectiveness for downstream tasks as we will show in our experiment.

There are also various graph embedding methods designed for specific graphs, like dynamic graphs \cite{dyngem,DyRep} and heterogeneous networks \cite{Metapath2vec,HIN2Vec}. In this paper, we focus on the most fundamental case when the network is static and no feature vector is given.

\section{{\lerand} Framework}
\label{sec:sec-Learning-based-method}

\begin{algorithm}[t]
\DontPrintSemicolon
    $\vect{\prox}\leftarrow \vect{0}, \vect{R}\leftarrow \vect{I}_n$\;
    \For{$k = 0$ to $L$}{
        $\vect{\prox} \leftarrow \vect{\prox} + \alpha_k \cdot \vect{R}$\;
        $\vect{R} \leftarrow (1-\alpha_k)\cdot \vect{P}\cdot \vect{R}$\;
    }
    \Return $\vect{\prox}$
\caption{Compute-Supervised-PPR($\vect{P}$, $L$)}
\label{alg:alg-compute-exact-sppr}
\end{algorithm}

In this section, we present our {\lerand} framework. Section \ref{subsec-proximity} introduces the trainable proximity matrix and the training parameters. Section \ref{subsec-training} elaborates on the training process of {\lerand} and introduces the loss functions used for link prediction and node classification, respectively. Section \ref{subsec-sampling} presents how to carefully obtain the subsampled graphs to do training on large graphs.

\subsection{Trainable Proximity Measure}
\label{subsec-proximity}

Recap from Section \ref{sec:sec-Introduction} that, the personalized PageRank (PPR) $\vect{\pi}_u(v)$ of node $v$ with respect to $u$ is the probability that an $\alpha$-discounted random walk from $u$ stops at node $v$, where an $\alpha$-discounted random walk has $\alpha$ probability to stop at the current node and $(1-\alpha)$ probability to randomly jump to one of its out-neighbors. Define the transition matrix $\vect{P}=\vect{D}^{-1}\vect{A}$ where $\vect{D}$ is the diagonal matrix such that $\vect{D}(i,i)$ is the out-degree (resp. degree) of node $i$ if $G$ is directed (resp. undirected), and $\vect{A}$ is the adjacency matrix. Then, the PPR proximity matrix can be expressed as:
$$
\vect{\prox} = \sum_{l=0}^\infty \alpha\cdot (1-\alpha)^l \cdot \vect{P}^l.
$$
In our trainable proximity measure, instead of fixing the stopping probability at each step to be $\alpha$, we allow the stopping probability $\alpha_l$ of the random walk at the $l$-th hop to be trainable. Currently, we assume that the stopping probability $\alpha_l$ at the $l$-th step is given and will show how to train $\alpha_l$ in Section \ref{subsec-training}. 
A random walk that follows such learned stopping probability at each step is denoted as a {\em supervised random walk} and the proximity derived from the supervised random walk is denoted as supervised PPR.
The supervised PPR proximity matrix can be defined as:
\begin{equation}
    \label{eq-proximity}
  \vect{\prox} = \alpha_0 \vect{I}_n + \sum_{l=1}^\infty \alpha_l\cdot \prod_{k=0}^{l-1}(1-\alpha_k) \cdot \vect{P}^l.    
\end{equation}
To explain, the probability that a supervised random walk stops at exactly the $l$-th hop is $\alpha_l\cdot \prod_{k=0}^{l-1}(1-\alpha_k)$. Thus, by summing up the probability to stop at each hop, we derive the supervised PPR score.

The exact supervised PPR score then can be computed with an iterative manner as shown in Algorithm \ref{alg:alg-compute-exact-sppr} when $L\rightarrow \infty$. However, we observe that when $L$ is sufficiently large, the probability that the supervised random walk stops with a hop number larger than $L$ is close to zero. Hence, we discard all supervised random walks that stop with a hop larger than $L$. The following theorem shows the quality of calculated supervised PPR scores with Algorithm \ref{alg:alg-compute-exact-sppr}. 

\begin{theorem}
\label{thm-sppr-error-rate}
Let $\vect{\prox}_L$ be the supervised PPR proximity matrix derived by Algorithm \ref{alg:alg-compute-exact-sppr}. Let $||\vect{M}||_{\infty}$ be the infinity-norm of matrix $\vect{M}$, we have:
$$
||\vect{\prox}_L -\vect{\prox}||_{\infty}\leq \prod_{k=0}^{L}(1-\alpha_k) || \vect{\prox}||_{\infty}.
$$
\end{theorem}

All the proofs of our theorems can be found in the appendix. According to our observation, the probability that the supervised random walk stops at a hop greater than $L$ is almost zero when we set $L=15$. Hence, $L$ is set to $15$ in our experiment.

Algorithm \ref{alg:alg-compute-exact-sppr} makes a tight connection between the supervised PPR proximity matrix and our training parameters $\alpha_l$ ($0\leq l \leq L$), which allows the gradients to flow from the derived supervised PPR proximity matrix $\vect{\prox}_L$ to the training parameters via backpropagation. Let $n$ and $m$ denote the number of nodes and the number of edges, respectively. The time complexity of the algorithm can be bounded by $O(m\cdot n \cdot L)$ since $\vect{P}$ is a sparse matrix with $m$ entries.

\header
{\bf Remark.} Note that the idea to learn the stopping probability for each hop is not a new idea in graph neural networks, e.g., \cite{dif,adaptppr}. It is much easier for neural networks as the stopping probabilities can be treated as additional weights to learn. However, it is more challenging to integrate this idea to node embedding, especially for matrix factorization methods. We show an efficient and effective solution that works for large graphs, which is non-trivial.

\subsection{Training process of {\lerand}}
\label{subsec-training}

Algorithm \ref{alg:alg-sum} shows the pseudo-code of the training process of {\lerand}. Firstly, it initializes the stopping probability $\alpha_k$ for $0\leq k \leq L$ such that the probability a random walk stops at the $l$-th hop follows some standard distribution, e.g., uniform, geometric, or Poisson (Line 1). Next, it uses the initial settings of these stopping probabilities to derive the supervised PPR score by invoking Algorithm \ref{alg:alg-compute-exact-sppr} (Line 3). Given the supervised PPR scores, it eliminates all the proximity scores that are too small. A threshold $\delta$ is included and all scores smaller than $\delta$ are set to zero (Lines 4-5). Then, a non-linear operation, $\log({\frac{\vect{\prox}}{\delta}})$, is applied to the proximity matrix $\vect{\prox}$. The proximity scores are divided by $\delta$ to guarantee that each entry after taking the $\log$ will be non-negative (Line 6). 
Thereafter, a differentiable SVD, e.g., \cite{low-rank}, is applied on the new matrix $\vect{M}$ with input parameter $d$ (Line 7). The SVD obtains two $n\times d$ matrices $\vect{U}$ and $\vect{V}$, and a $d\times d$ diagonal matrix $\vect{\Sigma}$ such that $\vect{U}\vect{\Sigma}\vect{V^T}\approx \vect{M}$. Given the three matrices, $\vect{U}\sqrt{\Sigma}$ is returned as the first embedding matrix $\vect{X}$ and $\vect{V}\sqrt{\Sigma}$ is returned as the second embedding matrix $\vect{Y}$.

Subsequently, the two embedding matrices are fed into the loss function for the link prediction or node classification (Line 9). The loss functions will be discussed shortly. Then, the stopping probabilities are updated according to the loss function by backpropagation. The training process terminates until the loss function converges (Line 2). Notice that, due to the extremely long computational chain, it is infeasible to write down the explicit form of the gradients. However, like modern deep neural networks, we can use the autograd feature in PyTorch to numerically compute the gradients with respect to the training parameters. Also, the builtin SVD in PyTorch supports to compute the gradient and hence we directly adopt the PyTorch implementation. In what follows, we elaborate on the details of the loss functions.

\header
\textbf{Loss function for link prediction.} For link prediction, there are two components in our loss function. The first component of the loss function aims to ensure that the total information in the supervised random walks started from node $u$ is closed to its out-degree. Let $\vect{\alpha}=(\alpha_0,\alpha_1,\cdots, \alpha_L)$. The first part of the loss function is as follows:

\begin{equation}
\label{loss_link_1}
    \mathcal{L}_1(\vect{\alpha};\vect{A}) = \frac{1}{n^2} \sum_u \parallel \sum_{v \neq u}  \vect{x}_u \cdot \vect{y}_v - d_{out}(u) \parallel_2^2, 
\end{equation}
where $\vect{x}_u$ is the $u$-th row of $\vect{X}$, $\vect{y}_v$ is the $v$-th row of $\vect{Y}$, and $d_{out}(u)$ is the out-degree of $u$. Equation \ref{loss_link_1} indicates that our learned embedding by the supervised random walk should preserve the out-degree information as much as possible.

The second part of the loss function for link prediction is the average cross-entropy over all existing edges:
\begin{equation}
\label{loss_link_2}
    \mathcal{L}_2(\vect{\alpha};\vect{A}) = - \frac{1}{m} \sum_u \sum_v \vect{A}_{u,v}\log(\sigma(\vect{x}_u \cdot \vect{y}_v)),
\end{equation}
where $\sigma(x) = 1/(1+\exp(-x))$ is the sigmoid function and $\vect{A}_{u,v}=1$ if edge $(u,v)$ exists in the input graph and $\vect{A}_{u,v}=0$ otherwise. 
The final loss function for link prediction is:
\begin{equation}
\label{loss_link}
    \mathcal{L}_{\textrm{p}} = \beta\mathcal{L}_1 + \gamma\mathcal{L}_2,
\end{equation}
where $\beta$ and $\gamma$ are two balancing hyperparameters.

\header
\textbf{Loss function for node classification.}
There are also two parts in the loss function for node classification. We first concatenate two output embedding matrices $\vect{X}$ and $\vect{Y}$ together to get the unique embedding matrix $\vect{Z} = \textrm{concat}(\vect{X},\vect{Y})$. Then following the fact that two randomly selected nodes have different labels with high probability, we randomly sample a small set of negative node-pairs $\mathcal{N} \subset V \times V$ for each iteration. Let $H=(V,\mathcal{N})$ denotes the graph with edge set $\mathcal{N}$ with unnormalized Laplacian matrix $\vect{L}_H$ and $G_k=(V_k,E_k)$ be a complete graph formed by nodes in $G$ with the same label $k$. Inspired by the relaxation proposed in \cite{sce}, the first loss function for node classification is defined as follows:

\begin{algorithm}[t]
\DontPrintSemicolon
    \KwInput{Matrix $\vect{P}$, maximum length $L$, threshold $\delta$, embedding dimension $d$, learning rate $\eta$}
    \KwOutput{stopping probabilities $\alpha_0,...,\alpha_L$}
    Initialize $\alpha_k$ for $k=0,1,...,L$\;
    \While{not convergence}{
        $\vect{\prox} \leftarrow$ Compute-Supervised-PPR($\vect{P}, L$)\;
        \If{$S(i,j) < \delta$, \textbf{for} $\forall S(i,j) \in \prox$}{
            $S(i,j) \leftarrow 0$\;
        }
        Get matrix $\vect{M} \leftarrow $ log $(\frac{\vect{\prox}}{\delta})$ for non-zero entries\;
        $[\vect{U},\vect{\Sigma},\vect{V}] \leftarrow $ Differentiable-SVD$(\vect{M},d)$\;
        $\vect{X} \leftarrow \vect{U}\sqrt{\vect{\Sigma}}$, $\vect{Y} \leftarrow \vect{V}\sqrt{\vect{\Sigma}}$\;
        Compute link prediction loss $\mathcal{L}_p$ via Eq.\ref{loss_link} or node classification loss $\mathcal{L}_c$ via Eq.\ref{loss_class}\;
        \For{$k=0,...,L$}{
            $\alpha_k \leftarrow \alpha_k - \eta \nabla_{\alpha_k}  \mathcal{L}$\;
        }
    }
    \Return $\alpha_0,...,\alpha_L$\;
\caption{\lerand-Trainining}
\label{alg:alg-sum}
\end{algorithm}

\begin{equation}
\label{loss_class_1}
    \mathcal{L}'_1(\vect{\alpha};\vect{A}) = \frac{\sum_{k=1}^{n_c} \sum_{u,v\in V_k} \left \| \vect{z}_u-\vect{z}_v \right \|_2^2}{n_c \cdot \sum_{(u,v)\in \mathcal{N}} \left \| \vect{z}_u-\vect{z}_v \right \|_2^2} = \frac{\sum_{k=1}^{n_c} \text{Tr}(\vect{Z}^{\top} \vect{L}_k \vect{Z})}{n_c \text{Tr}(\vect{Z}^{\top} \vect{L}_H \vect{Z})},
\end{equation}
where $\vect{z}_u$ is the concatenated representation of node $u$, $n_c$ is the total number of class labels in the graph, and $L_k$ is the unnormalized Laplacian matrix of $G_k$. The goal of $\mathcal{L'}_1$ is to minimize pair-wise distances between node pairs with the same class label and maximize pair-wise distances between negative samples. 

Next, we employ an activation function to normalize the output embedding to a probability distribution over predicted class labels: $p_{uk}={\textrm{softmax}}(\vect{Z}\cdot \vect{W}+\vect{b})_{uk}$, where $\vect{W}$ is a $2d \times n_c$ fixed mapping matrix generated from uniform distribution, $\vect{b}$ is the bias term, $p_{uk}$ denotes the probability that node $u$ has class label $k$, and ${\textrm{softmax}}(\vect{x})_{uk} = \exp(\vect{x}_{uk})/(\sum_{c=1}^{n_c} \exp(\vect{x}_{uc}))$. Let $\mathcal{Y}$ denote the $n \times n_c$ label matrix, where $\mathcal{Y}_{u,k}=1$ if node $u$ has class label $k$ and $\mathcal{Y}_{u,k}=0$ otherwise. Then, the second part of the loss function for node classification is the average cross-entropy over all nodes:
\begin{equation}
\label{loss_class_2}
    \mathcal{L}'_2(\vect{\alpha};\vect{A}) = -\frac{1}{n}\sum_{u}\sum_{k=1}^{n_c} \mathcal{Y}_{u,k}\log p_{uk}.
\end{equation}

The final loss function for node classification is defined according to $\mathcal{L}'_1$ and $\mathcal{L}'_2$ as follows:
\begin{equation}
\label{loss_class}
    \mathcal{L}_{\textrm{c}} = \beta'\mathcal{L}'_1 + \gamma'\mathcal{L}'_2,
\end{equation}
where $\beta'$ and $\gamma'$ are two balancing hyperparameters. 

\header
{\bf Remark.} Notice that the training algorithm for {\lerand} on node classification needs additional label information and we randomly sample $5\%$ of the labeled nodes for training.
To make a fair comparison with our competitors, to train the classifiers, we will only include $(x-5)\%$ new training data if we split $x\%$ of the data for training and the remaining $(100-x)\%$ for testing. This guarantees that our {\lerand} only accesses the same amount of labeled data compared to other competitors.

\subsection{Training {\lerand} with sub-sampling}
\label{subsec-sampling}

The above-mentioned training process requires calculations on a dense matrix of $O(n^2)$ size and requires $O(L\cdot n\cdot m)$ running cost, which makes it non-scalable to large graphs. However, such cost seems unavoidable at first glance since we need to obtain the proximity matrix to do backpropagation. After a careful analysis, we make the following two observations to help us avoid the high running costs.  Our first key observation is that the parameters that we need to train are only the stopping probabilities at each hop, which are node-independent. That is to say, if we can find a subgraph of the input graph $G$ such that the learned stopping probabilities on the subgraph are identical to that on $G$ for the same task, we can simply learn the parameters on the subgraph with smaller size and apply the learned parameters to the original graph directly, reducing the computational costs. Another observation is that a subgraph of $G$ with similar connectivity should share similar learning stopping probabilities as the input graph on the same task. 

Motivated by this, we present our sub-sampling based training method for {\lerand} on large graphs. Obviously, a straightforward solution is to sample a number $n_s$ of nodes and then consider the subgraph containing these $n_s$ nodes. However, such a solution severely degrades the connectivity among the nodes. Simple edge sampling strategies will face a similar dilemma which hampers the connectivity among the sampled nodes. 

\begin{algorithm}[t]
\DontPrintSemicolon
    \KwInput{Graph \textit{G}, source node $s$, threshold $\delta$, stopping probabilities $\alpha_0,...,\alpha_L$}
    \KwOutput{Approximate proximity vector $\hat{\vect{\pi}}_s$}
    Initialize $\hat{\vect{\pi}}_s \leftarrow 0, \vect{r}^{(k)}_s \leftarrow 0 $ for $k=0,1,...,L$\;
    Initialize $\vect{r}^{(0)}_s(s) \leftarrow 1$\;
    \While{$\exists v \in V, 0 \leq k \leq L$ such that $\vect{r}^{(k)}_s(v) > \delta \cdot d_{out}(v)$}{
        $\hat{\vect{\pi}}_s(v) \leftarrow \hat{\vect{\pi}}_s(v) + \alpha_k \cdot \vect{r}^{(k)}_s(v)$\;
        \For{$u \in N(v)$}{
            $\vect{r}^{(k+1)}_s(u) \leftarrow \vect{r}^{(k+1)}_s(u) + (1-\alpha_k)\cdot \frac{\vect{r}^{(k)}_s(v)}{d_{out}(v)}$\;
        }
        $\vect{r}^{(k)}_s(v) \leftarrow 0$\;
    }
    \Return $\hat{\vect{\pi}}_s$\;
\caption{{\lerand}-Generalized-Push}
 \label{alg:alg-push}
\end{algorithm}

To keep the connectivity among the sampled nodes, we apply a BFS style traversal for subgraph sampling. Our goal is still to sample a subgraph with a constant number $n_s$ of nodes. To sample such a subgraph, firstly a source node $u$ is randomly sampled from the input graph $G$. Thereafter, a BFS traversal is applied from the source $u$ to explore the local community of node $u$. If the number of visited nodes by the BFS from $u$ is smaller than $n_s$, another node $v$ is randomly sampled as the source to do BFS. The BFS sampling stops as soon as in total $n_s$ nodes are visited. 

However, the weights trained on a single subgraph might be biased and makes the learned stopping probabilities non-generalizable to the original input graph $G$. To make the learned parameters generalizable to the input graph, we sample multiple subgraphs by the above strategy to learn the parameters. In particular, in each iteration, we sample a subgraph by the BFS strategy and then update the training parameters, i.e., the stopping probabilities, according to the loss functions defined on this subgraph. 
The loss functions on the subgraph are modified accordingly where the total number $n$ of nodes in Equations \ref{loss_link_1} and \ref{loss_class_2} are replaced by sample size $n_s$; the total number $m$ of edges in Equation \ref{loss_link_2} is also replaced by sample size $n_s$; the set of nodes with label $k$ in Equation \ref{loss_class_1} is changed to $V_{Sk}= V_k \cap V_S$, where $V_S$ is the node set of the subgraph, $V_k$ is the set of nodes with label $k$ in $G$; and the negative sample set in Equation \ref{loss_class_1} is changed to $\mathcal{N_S} \subset V_S \times V_S$.

With such a sampling technique, the time complexity of Algorithm \ref{alg:alg-sum} can be bounded by $O(h\cdot L\cdot n_s^3)$ where $h$ is the number of training iterations by Algorithm \ref{alg:alg-sum} until it converges. Since $n_s$ is a controllable constant, we set $n_s$ such that the proximity matrix can be fed into the GPU memory for more efficient training.

\section{Generalized Push}
\label{sec:sec-gpush}

Given the learned stopping probabilities, a straightforward solution is to invoke Algorithm \ref{alg:alg-compute-exact-sppr} to derive the supervised PPR scores. However, this incurs $O(L\cdot n\cdot m)$ running cost, which is prohibitive for large graphs. To tackle this issue, we present a generalized push algorithm to efficiently compute the supervised PPR proximity matrix with $O(\frac{n}{\delta})$ cost, where $\delta$ is a parameter to control the computational cost as well as the sparsity of the proximity matrix. 

\subsection{Generalized Push Algorithm}

Our main idea to compute the supervised PPR proximity matrix is to derive a sparsified proximity matrix such that the supervise PPR scores no larger than $\delta$ can be safely discarded. But still, how much cost should we take to derive a sufficiently accurate approximation proximity score? In STRAP \cite{STRAP}, the authors propose a solution to derive the PPR estimations such that 
$
|\vect{\pi}_u(v)  - \hat{\vect{\pi}}_u(v)| < \delta
$
with a cost of $O(\frac{m}{\delta})$. But, can we further reduce the running cost without sacrificing the embedding quality? Here, we give an affirmative answer. Our solution is inspired by the local graph clustering algorithm Local-Push \cite{push} which returns approximate PPRs with respect to a source $s$ in $O(\frac{1}{\delta})$ running time and guarantees that
$$
|\vect{\pi}_u(v)  - \hat{\vect{\pi}}_u(v)|/d_{out}(v) < \delta \textrm{, for any } v \in V,
$$
on undirected graphs where $d_{out}(v)$ is the degree of node $v$. The Local-Push algorithm suggests that if we only want to compute the approximate PPR scores around the local graph cluster with respect to a source, the running time can be reduced. In our case, the node embedding aims to find nodes that are their representatives and the nodes in their local graph cluster stand as perfect representatives. However, the Local-Push algorithm only works for PPR, not for our supervised PPR. Thus, we present a generalized push algorithm that works for arbitrary stopping probabilities.

Algorithm \ref{alg:alg-push} shows the pseudo-code for our generalized push algorithm. Given a source node $s$, a vector $\vect{\hat{\pi}}_s$ is maintained to store the portion of supervised random walks that has stopped at each node, and is the estimated supervised PPR scores with respect to source $s$. Besides, for each hop $0\leq k\leq L$, where $L$ is the maximum length of a truncated supervised random walk introduced in Section $\ref{subsec-proximity}$, an additional residue vector $\vect{r}_s^{(k)}$ is maintained. The vector $\vect{r}_s^{(k)}$ indicates the portion of supervised random walks from $s$ that currently stay at the $k$-th hop but have not stopped yet. Thus, if the residue vectors are all zero, it returns the exact supervised PPR scores. Initially, the residue vectors are all zero except for $\vect{r}_s^{(0)}(s)=1$ (Lines 1-2), indicating that the supervised random walks initially all stay at $s$ and has not stopped yet. Then, if any entry $\vect{r}_s^{(k)}(v)$ in the $L$ residue vectors is above $\delta\cdot d_{out}(v)$ (Line 3), a push operation (Lines 4-7) is invoked. In particular, it first converts $\alpha_k\cdot \vect{r}_s^{(k)}(v)$ to $\vect{\hat{\pi}}_s(v)$ (Line 4). To explain, $\alpha_k$ portion of the $\vect{r}_s^{(k)}(v)$ random walks stop at the $k$-th hop. Next, the remaining $(1-\alpha_k)$ portion of the $\vect{r}_s^{(k)}(v)$ random walks randomly jump to the out-neighbors of $v$ (Lines 5-6). Thus, for each $u$ that is an out-neighbor of $v$, the residue $\vect{r}_s^{(k+1)}(u)$ is incremented by $(1-\alpha_k)\cdot \vect{r}_s^{(k)}(v)/d_{out}(v)$. After the push operation, the residue $\vect{r}_s^{(k)}(v)$ is set to zero (Line 7). The algorithm terminates when there exists no residue $\vect{r}_s^{(k)}(v)$ for any $k$ such that it is larger than $d_{out}(v)\cdot \delta$.

\begin{algorithm}[t]
\DontPrintSemicolon
    \KwInput{Graph \textit{G}, dimension $d$, threshold $\delta$, stopping probability vector $\vect{\alpha}$} 
    \KwOutput{Embedding matrices $\vect{X}$ and $\vect{Y}$}
    Initialize the proximity matrix $\vect{S} \leftarrow \vect{0}$\;
    \For{each node $u \in V$}{
        $\vect{\hat{\pi}}_u \leftarrow$ {\lerand}-Generalized-Push$(G,u,\delta,\vect{\alpha})$\;
        $\vect{\hat{\pi}}_u^T \leftarrow$ {\lerand}-Generalized-Push$(G^T,u,\delta,\vect{\alpha})$\;
        \For{each node $v$ in $V$}{
            \If{$\vect{\hat{\pi}}_u(v) > \delta$}{
                $\vect{S}(u,v) += \vect{\hat{\pi}}_u(v)$\;
            }
            \If{{$\vect{\hat{\pi}}^T_u(v) > \delta$}}{
                $\vect{S}(v,u) += \vect{\hat{\pi}}^T_u(v)$
            }
        }
    }
    $\vect{M} \leftarrow \log (\frac{\vect{\prox}}{\delta})$ for non-zero entries\;
    $[\vect{U},\vect{\Sigma},\vect{V}] \leftarrow $ SparseSVD$(\vect{M},d)$\;
    $\vect{X} \leftarrow \vect{U}\sqrt{\vect{\Sigma}}$, $\vect{Y} \leftarrow \vect{V}\sqrt{\vect{\Sigma}}$\;
    \Return $\vect{X}, \vect{Y}$\;
\caption{\lerand-Embedding}
\label{alg:alg-gpt}
\end{algorithm}

\begin{theorem}\label{thm:gpush-time-complexity-1}
Algorithm \ref{alg:alg-push} runs in $O(\frac{1}{\delta})$ time.
\end{theorem}

\begin{theorem}\label{thm:gpush-time-complexity-2}
Let $\vect{\pi}_s^{L}(u)$ be the supervised PPR considering random walks within $L$ hops. Then for undirected graphs, Algorithm \ref{alg:alg-push} returns an estimation $\vect{\hat{\pi}}_s(u)$ of $\vect{\pi}_s^{L}(u)$ for each node $u$ such that:
$$
|\vect{\hat{\pi}}_s(u)-\vect{\pi}^L_s(u)|/d_{out}(u) \leq \delta\cdot L.
$$
\end{theorem}

\noindent
By setting $\delta=\frac{\delta'}{L}$, Algorithm \ref{alg:alg-push} runs in $O(\frac{L}{\delta'})$ time. At the same time, the error bound in Algorithm \ref{thm:gpush-time-complexity-2} can be bounded by $L\cdot \frac{\delta'}{L}=\delta'$. Since $L$ can be treated as a constant, the running time with $\delta=\frac{\delta'}{L}$ is still $O(\frac{1}{\delta'})$ and for any node $u$, we have that: 
$$
|\vect{\hat{\pi}}_s(u)-\vect{\pi}^L_s(u)|/d_{out}(u) \leq \delta'.
$$

The above analysis shows that our generalized push algorithm can provide identical result quality as the Local-Push algorithm with the identical asymptotic running cost, thus returning high-quality results for the representative nodes of the source node.

\subsection{Final Embedding}
Given the generalized push algorithm, we finally show how to output the embedding for the graph. Following STRAP \cite{STRAP}, we compute the supervised PPR on both the input graph $G$ and the transpose graph $G^T$ by reversing the direction of each edge of $G$ and set $\vect{S}(u,v)$ as $\vect{\pi}_u(v)+\vect{\pi}_v^T(u)$, where $\vect{\pi}_u(v)$ is the supervised PPR of $v$ with respect to $u$ and $\vect{\pi}_v^T(u)$ is the supervised PPR of $u$ with respect to $v$ on the transpose graph $G^T$. Note that we do not bring this part into our training phase to reduce the computational costs since we use the same stopping probabilities for both the input graph $G$ and the transpose graph $G^T$. Algorithm \ref{alg:alg-gpt} shows how to compute approximate proximity scores. For each node $u$, we compute the approximate supervised PPR on the input graph $G$ and the transpose graph $G^T$ (Lines 3-4). For any approximate supervised PPR score $\vect{\hat{\pi}}_u(v)$, it is added to $\vect{S}(u,v)$ only if it is larger than the threshold $\delta$. Similarly, each approximate supervised PPR score $\vect{\hat{\pi}}^T_u(v)$ on the transpose graph $G^T$ is added to $\vect{S}(v,u)$ only if $\vect{\hat{\pi}}^T_u(v)$ is larger than $\delta$ (Lines 5-9). With such a pruning strategy, the proximity matrix $\vect{S}$ is sparsified to include $O(\frac{n}{\delta})$ non-zero entries. Then, a non-linear operation $\log(\frac{\vect{S}}{\delta})$, is applied to $\vect{S}$. Notice that all the entries are divided by $\delta$ before taking the logarithm to guarantee that the values will be non-negative (Line 10). The resulting matrix $\vect{M}$ is then fed to a sparse SVD to derive final embedding matrices $\vect{X}$ and $\vect{Y}$ (Lines 11-12). We have the following theorem for the running time and decomposition quality with respect to $\vect{M}$. 
\begin{theorem}
\label{thm:complexity-of-Lemane}
Algorithm \ref{alg:alg-gpt} runs in $O(\frac{n}{\delta}+\frac{n\cdot d^2}{\epsilon^4})$ time to guarantee that the embedding preserves the feeding matrix $\vect{M}$ with $(1+\epsilon)$-approximation to the best rank-$d$ matrix in terms of Frobenius-norm.
$$
||\vect{M}- \vect{X}\vect{Y}^T||_F \leq (1+\epsilon) \min_{rank(\vect{B})\leq d} ||\vect{M}-\vect{B}||_F.
$$
\end{theorem}
Following previous work \cite{STRAP}, we treat $\epsilon$ as a constant and use the default setting of builtin SVD implementations. The running time of Algorithm \ref{alg:alg-gpt} is $O(\frac{n}{\delta}+n\cdot d^2)$.

\section{Experiments}
\label{sec:sec-Experiment}

\begin{table}[t]
  \caption{Dataset statistics.}
  \vspace{-4mm}
  \label{tab:table-data-sets}
  \begin{tabular}{llccc}
    \toprule
    Name & Type& $n$ & $m$  & labels\\
    \midrule
    Wikipedia  &   directed&   4.78K& 184.81K& 40 \\
    WikiVote   &   directed&   7.12K& 103.69K&  - \\
    BlogCatalog& undirected&  10.31K& 333.98K& 39 \\
    Slashdot   &   directed&  82.17K& 870.16K&  - \\
    TWeibo     &   directed&   1.94M&  50.66M& 100\\
    Orkut      & undirected&   3.07M& 117.19M& 100\\
    \bottomrule
\end{tabular}
\end{table}

\begin{table}[t]
\vspace{-1mm}
  \caption{Link prediction precision (\%) on small datasets.}
  \vspace{-4mm}
  \label{tab:table-link1}
  \begin{tabular}{lccc}
    \toprule
    Method & Wikipedia& Wikivote& BlogCatalog\\
    \midrule
    DeepWalk & \cellcolor{LightSteelBlue}88.33& 68.32& 85.35\\
    Node2vec & 82.86& 78.50& 82.07\\
    VERSE    & \underline{88.09}& 82.82& 88.49\\
     InfiniteWalk  & 66.86& 81.07& 84.05\\
    \hline
    AROPE    & 84.27& 62.08& 88.30\\
    RandNE   & 83.15& 77.62& 87.09\\
    ProNE    & 52.06& 66.22& 57.87\\ 
    NetSMF   & 72.01& 72.64& 47.92\\
    STRAP    & 86.53& \underline{92.58}& 89.58\\
    NRP      & 83.56& 91.07& \cellcolor{LightSteelBlue}90.10\\
    \hline
    GraphGAN & 70.33& 71.76& 71.83\\
    AW       & 50.42& 56.62& 62.56\\
    \hline
    NetHiex  & 45.03& 73.01& 65.58\\
    GraphZoom& 84.73& 82.10& 86.82\\ 
    Louvain  & 56.22& 58.33& 59.97\\
    \midrule
    {\lerand}-F & 87.79& \cellcolor{LightSteelBlue}92.78& \underline{89.92}\\
    \bottomrule
\end{tabular}
\end{table}

We compare our {\lerand} against alternative solutions on link prediction and node classification tasks. All experiments are conducted on a Linux machine with an Intel Xeon(R) CPU clocked at 2.70GHz, an NVIDIA GeForce RTX 2080 Super 8GB GPU, and 384GB memory.

\subsection{Experimental settings}

\textbf{Datasets.}
We test on six real datasets that are used in recent node embedding studies \cite{deepwalk,node2vec,verse,pbg,STRAP,tea,nrp}. The statistics of these datasets are shown in Table \ref{tab:table-data-sets}.
{\em BlogCatalog} \cite{blogcatalog-flickr},
{\em Slashdot} \cite{slashdot}, {\em TWeibo} \cite{tweibo} and {\em Orkut} \cite{orkut} are four social networks in which links represent a friendship/following relationship between users.
{\em Wikipedia} \cite{wiki} is a co-occurrence network of words appearing in the Wikipedia dump.
{\em Wikivote} \cite{wikivote} is a who-votes-on-whom network on Wikipedia. All datasets and the node labels (if any) can be downloaded from public sources \cite{blog,wiki,snap,tweibo}.

\header
\textbf{Competitors.}
We evaluate {\lerand} against 15 node embedding methods, including some classic methods and several state-of-the-art methods. We divide these methods into four groups as follows.

\begin{itemize}[topsep=0mm, partopsep=1pt, itemsep=0pt, leftmargin=9pt]
    \item Skip-gram methods: DeepWalk \cite{deepwalk}, Node2vec \cite{node2vec}, VERSE \cite{verse}, and InfiniteWalk \cite{infwalk};
    \item Matrix factorization methods: AROPE \cite{AROPE}, RandNE \cite{RandNE}, ProNE \cite{ProNE}, NetSMF \cite{NetSMF}, STRAP \cite{STRAP}, and NRP\footnote{There were some implementation issues in the released code of NRP on undirected graphs, which was fixed recently by the inventors.} \cite{nrp};
    \item Neural network methods: GraphGAN \cite{GraphGAN} and AW \cite{AW};
    \item Other methods: NetHiex \cite{nethiex}, GraphZOOM\footnote{We use the default embedding method DeepWalk for GraphZoom.} \cite{graphzoom}, Louvain \cite{Louvain}.
\end{itemize}

\noindent
For our methods, we use {\lerand-F} to indicate the algorithm trained on the entire graph and {\lerand-S} to indicate the algorithm trained on subsampled graphs. Notice that {\lerand}-F is adopted for small datasets Wikipedia, Wikivote, and BlogCatalog. {\lerand}-S is adopted for large datasets Slashdot, TWeibo, and Orkut.

\begin{table}[t]
  \caption{Link prediction precision (\%) on large datasets.}
  \vspace{-4mm}
  \label{tab:table-link2}
  \begin{tabular}{lccc}
    \toprule
    Method   & Slashdot&  Tweibo &  Orkut \\
    \midrule
    AROPE    &  82.83  &  69.46  &  82.03 \\
    RandNE   &  81.03  &  70.74  &  79.45 \\
    ProNE    &  72.80  &  45.47  &  80.88 \\
    STRAP    &  \underline{83.07}  &  \underline{94.58}  &  85.73 \\
    NRP      &  80.98  &  93.87  &  \underline{86.34} \\
    Louvain  &  55.56  &  64.25  &  80.85 \\
    \midrule 
    {\lerand}-S & \cellcolor{LightSteelBlue}84.13 & \cellcolor{LightSteelBlue}94.89 & \cellcolor{LightSteelBlue}89.15\\
    \bottomrule
\end{tabular}
\end{table}

\header
\textbf{Parameter settings.}
We obtain the source code of all competitors from GitHub and perform these methods with default parameter settings suggested by their authors. Following previous studies \cite{deepwalk,verse,STRAP}, we set the embedding dimensionality $d=128$. For our {\lerand}-S, we set the sample size $n_s=5000$. For {\lerand}-F and {\lerand}-S, we use grid search to set $\beta, \beta', \gamma, \gamma'$ from \{0.01, 0.1, 0.5, 1, 2, 3\}, learning rate from \{0.001, 0.005, 0.01, 0.05, 0.1, 0.5\}, and threshold $\delta$ from \{$10^{-7}$,$10^{-6}$,$10^{-5}$,$10^{-4}$\}; we use JacobiSVD for {\lerand}-F and frPCA \cite{frpca} for {\lerand}-S, to generate final embeddings.

Since the backpropagation in {\lerand} is complicated, our objective function may easily fall into a local minimum. To tackle this issue, the stopping probabilities are initialized with different distributions. Specifically, we set $\alpha_0,\alpha_1,\cdots, \alpha_L$ such that the probability that the supervised random walk stops at the $l$-th hop follows a geometric distribution or Poisson distribution with respect to $l$ and we report the best results of {\lerand}. The parameters of {\lerand} are optimized by Stochastic Gradient Descent (SGD) optimizer. 

\subsection{Link Prediction}

Link prediction aims to predict which pairs of nodes are likely to form edges. Following previous work \cite{nrp,AROPE}, we randomly hide $30\%$ of the edges for testing and train the embedding vectors on the rest of the graph. Then, the testing set is generated by including {\em (i)} the node pairs corresponding to the $30\%$ removed edges, and {\em (ii)} an equal number of node pairs that are not connected by any edge in the original graph $G$. Given a node pair $(u,v)$ in $E_{test}$, we compute a score for $(u,v)$ based on the embedding vectors, and evaluate model performance using precision score.

Following previous work \cite{STRAP, nrp}, for DeepWalk, Node2vec, VERSE, InfiniteWalk, GraphGAN, and Louvain, we use the edge-feature approach introduced in \cite{nethiex}: {\em (i)} randomly select $30\%$ existing edges which are not in $E_{test}$ and the same number of non-existing edges as training set $E_{train}$ on each dataset; {\em (ii)} for each node pair $(u,v)\in E_{train} \cup E_{test}$, we concatenate $d$-dimension embedding vectors of node $u$ and that of  node $v$; {\em (iii)} we consider the $2d$-length vectors as the features of node pairs in $E_{train}$ and feed them into a binary logistic regression classifier; {\em (iv)} then the trained classifier is used to perform link prediction on $E_{test}$. For AROPE, RandNE, ProNE, AW, NetHiex, and GraphZOOM, the score of a node pair $(u,v)$ is the inner product of the embedding vector of node $u$ and that of node $v$; for NetSMF, STRAP, NRP, and our {\lerand}, the score of a node pair $(u,v)$ is the inner product of embedding $\vect{x}_u$ from $\vect{X}$ and $\vect{y}_v$ from $\vect{Y}$ (Ref. to Section \ref{sec:sec-Learning-based-method} for the definitions of $\vect{X}$ and $\vect{Y}$).

Table \ref{tab:table-link1} reports the performance of {\lerand}-F against the 15 competitors on three small datasets. Table \ref{tab:table-link2} further reports the performance of {\lerand}-S against 6 methods which scale to large graphs. For the other 9 methods, they either cannot finish training in 24 hours or run out of memory on the large graphs.
As we can observe, our {\lerand} shows the best performance on 4 social networks where link prediction finds extensive applications. On the Wikipedia dataset, a co-occurrence network of words appearing in the Wikipedia, our {\lerand} still achieves high performance and is the best method among all matrix factorization methods.
Compared to two state-of-the-art matrix factorization methods STRAP and NRP, our {\lerand} takes the lead by more than $1\%$ on Wikipedia and Slashdot, and takes the lead by almost $3\%$ on Orkut.
This demonstrates the effectiveness of our learning based method.

\begin{table}[t]
  \caption{Node Classification Micro-F1 (\%) on Wikipedia.}
  \vspace{-4mm}
  \label{tab:table-class-wiki-full}
  \begin{tabular}{lccccc}
    \toprule
    Method   &  0.1 &  0.3 &  0.5 &  0.7 &  0.9 \\
    \midrule
    DeepWalk & 42.02& 46.12& 48.46& 49.39& 49.35\\
    Node2vec & 44.75& 48.08& 49.82& 50.69& 50.41\\
    VERSE    & 38.76& 41.92& 43.84& 44.92& 44.31\\
    InfiniteWalk  & 38.96& 42.64& 45.94& 47.73& 48.24\\
    \hline
    AROPE    & 45.82& 50.59& 52.47& 53.36& 52.16\\
    RandNE   & 34.51& 32.83& 43.25& 45.55& 45.93\\
    ProNE    & 44.49& 50.45& 53.15& 54.38& \cellcolor{LightSteelBlue}{54.43}\\
    NetSMF   & 40.29& 42.56& 43.68& 44.08& 44.12\\
    STRAP    & 46.51& 50.77& 52.44& 52.64& 52.37\\
    NRP      & \cellcolor{LightSteelBlue}{48.04}& \cellcolor{LightSteelBlue}{52.71}& \cellcolor{LightSteelBlue}{54.39}& \cellcolor{LightSteelBlue}{55.20}& \underline{54.30}\\
    \hline
    GraphGAN & 32.87& 35.43& 36.47& 37.68& 37.50\\
    AW       & 40.70& 40.70& 40.44& 40.30& 39.47\\
    \hline
    NetHiex  & 45.58& 47.95& 49.39& 49.77& 49.20\\
    GraphZoom& 40.76& 41.03& 41.18& 41.07& 40.54\\
    Louvain  & 40.86& 40.99& 40.72& 41.25& 40.69\\
    \midrule
    {\lerand}-F & \underline{47.79}& \underline{52.39}& \underline{54.34}& \underline{54.67}& 54.25\\
    \bottomrule
    \vspace{-2mm}
\end{tabular}
\end{table}

\begin{table}[t]
  \caption{Node Classification Micro-F1 (\%) on BlogCatalog.}
  \vspace{-4mm}
  \label{tab:table-class-Blog-full}
  \begin{tabular}{lccccc}
    \toprule
    Method   &  0.1 &  0.3 &  0.5 &  0.7 &  0.9 \\
    \midrule
    DeepWalk & 33.01& 36.97& 38.70& 39.87& 41.31\\
    Node2vec & 35.01& 37.16& 37.97& 38.51& 39.13\\
    VERSE    & 32.76& 36.32& 38.18& 39.09& 40.71\\
     InfiniteWalk  & 34.30& 38.00& 40.21& 41.87& 43.14\\
    \hline
    AROPE    & 29.12& 32.78& 34.12& 34.95& 35.77\\
    RandNE   & 26.75& 31.56& 36.20& 38.34& 39.93\\
    ProNE    & 36.38& 40.33& 41.56& 42.32& 42.28\\
    NetSMF   & 34.95& 37.99& 39.30& 40.19& 40.72\\
    STRAP    & 38.62& \underline{41.80}& \underline{42.96}& \underline{43.39}& \underline{43.97}\\
    NRP      & \underline{38.73}& 41.65& 42.36& 43.15& 43.34\\
    \hline
    GraphGAN & 14.97& 17.23& 18.81& 20.07& 21.16\\
    AW       & 16.52& 16.91& 16.98& 17.25& 17.39\\
    \hline
    NetHiex  & 37.46& 40.06& 40.63& 41.43& 42.33\\
    GraphZoom& 22.02& 25.59& 27.75& 29.37& 30.70\\
    Louvain  & 19.16& 19.88& 21.12& 21.30& 21.74\\
    \midrule
    {\lerand}-F & \cellcolor{LightSteelBlue}{39.64}& \cellcolor{LightSteelBlue}{42.38}& \cellcolor{LightSteelBlue}{43.34}& \cellcolor{LightSteelBlue}{44.03} &\cellcolor{LightSteelBlue}{44.37}\\
    \bottomrule
    \vspace{-2mm}
\end{tabular}
\end{table}

\subsection{Node Classification}
Node classification task aims to predict the label(s) of each node based on the embeddings. As we mentioned in Section \ref{sec:sec-Learning-based-method}, we first randomly sample $5\%$ labeled nodes for parameter training. Then the classification task is performed with the following steps: {\em (i)} following \cite{nrp}, for {\lerand} and other matrix factorization methods \footnote{For matrix factorization based methods, there was some implementation issues in the evaluation code for node classification on directed graphs. We have rerun the experiment for matrix factorization based methods on directed graphs.}, we first normalize the embedding vector\footnote{We observe some significant improvement on Orkut dataset when normalization is applied. Thus, we take normalization for the embedding vectors on all datasets.} $\vect{x}_v$ from $\vect{X}$ and embedding vector $\vect{y}_v$ from $\vect{Y}$ of each node $v$, and then concatenate them to get the representation of $v$; {\em (ii)} for methods without factorization operation, we use the embedding vector of node $v$ as its representation. Specifically, we randomly split the node and label sets into the training set and testing set and the training ratio varies from $10\%$ to $90\%$. To make a fair comparison, note that if the training ratio is $x\%$, then for {\lerand}, it will sample an additional $(x-5)\%$
and include the $5\%$ labeled nodes to train the classifiers.
Following previous work \cite{deepwalk,node2vec}, we employ a one-vs-rest logistic regression classifier implemented by LIBLINEAR \cite{liblinear} with default parameters for all methods. Micro-F1 score is used as the evaluation metric for the classification task.

\begin{table}[t]
  \caption{Node classification Micro-F1 (\%) on TWeibo.}
  \vspace{-4mm}
  \label{tab:table-class-Tweibo}
  \begin{tabular}{lccccc}
    \toprule
    Method   &  0.1 &  0.3 &  0.5 &  0.7 &  0.9 \\
    \midrule
    AROPE    & 33.96& 34.11& 34.18& 34.24& 34.29\\
    RandNE   & 34.64& 34.67& 34.80& 34.92& 34.96\\
    ProNE    & 35.27& 35.37& 35.43& 35.48& 35.52\\
    STRAP    & \underline{35.75} &\underline{35.97} &\underline{36.03}& \underline{36.04}& \underline{36.04}\\
    NRP      & 35.73 &\underline{35.97} &\underline{36.03}& \underline{36.04}& \underline{36.04}\\
    Louvain  & 34.14& 34.29& 34.33& 34.38& 34.42\\
    \midrule
    {\lerand}-S & \cellcolor{LightSteelBlue}{35.77}& \cellcolor{LightSteelBlue}{36.03}& \cellcolor{LightSteelBlue}{36.04}& \cellcolor{LightSteelBlue}{36.05}& \cellcolor{LightSteelBlue}{36.05}\\
    \bottomrule
    \vspace{-2mm}
\end{tabular}
\end{table} 
\begin{table}[t]
  \caption{Node classification Micro-F1 (\%) on Orkut.}
  \vspace{-4mm}
  \label{tab:table-class-Orkut}
  \begin{tabular}{lccccc}
    \toprule
    Method   &  0.1 &  0.3 &  0.5 &  0.7 &  0.9 \\
    \midrule
    AROPE    & 49.11& 50.86& 51.55& 51.89& 52.22\\
    RandNE   & 44.94& 49.35& 50.47& 51.11& 51.53\\
    ProNE    & 36.09& 37.13& 38.32& 38.81& 39.44\\
    STRAP    & 70.37& 73.09& 74.04& 74.60& 75.08\\
    NRP      & \underline{72.47}& \underline{75.58}& \underline{76.69}& \underline{77.36}& \underline{77.98}\\
    Louvain  & 29.16& 36.00& 36.51& 36.85& 36.63\\
    \midrule
    {\lerand}-S &\cellcolor{LightSteelBlue}{73.32}& \cellcolor{LightSteelBlue}{76.27}& \cellcolor{LightSteelBlue}{77.26}& \cellcolor{LightSteelBlue}{77.89}& \cellcolor{LightSteelBlue}{78.14}\\
    \bottomrule
    \vspace{-2mm}
\end{tabular}
\end{table} 

For node classification, we test on four datasets Wikipedia, BlogCatalog, TWeibo, and Orkut, which include label information. Table \ref{tab:table-class-wiki-full} and Table \ref{tab:table-class-Blog-full} show the performance of our {\lerand}-F against all 15 methods on the two small datasets: Wikipedia and BlogCatalog, respectively. Table \ref{tab:table-class-Tweibo} and Table \ref{tab:table-class-Orkut} show the performance of our {\lerand}-S against 6 methods that scale to TWeibo and Orkut.

We make the following observations. Firstly, our {\lerand} achieves the best Micro-F1 scores on three datasets BlogCatalog, Tweibo, and Orkut in all of the tested training ratios. Besides, compared to STRAP, which takes PPR without training the stopping probabilities, our {\lerand} achieves more than $1\%$ lead on Wikipedia datasets and up to $3\%$ on the Orkut dataset. Compared to the second-best matrix factorization method NRP, our {\lerand} further achieves about $1\%$ lead on the BlogCatalog dataset in all of the tested training ratios.

In summary, the experimental study reveals that our learning based {\lerand} can learn proximity measures that most suit the task in most scenarios. Compared to other matrix factorization methods, e.g., STRAP \cite{STRAP} and NRP \cite{nrp}, that take personalized PageRank as the proximity measure without learning, our {\lerand} can train the proximity measure, i.e., the supervised PPR, to gain better performance.

\section{Conclusion}
\label{sec:sec-Conclusion}
In this paper,  we present {\lerand} that learns trainable proximity measures to best suit the datasets and tasks at hand automatically. Experimental results reveal that {\lerand} can learn more representative embeddings compared with state-of-the-art approaches.

\balance

\begin{acks}
Sibo Wang is supported by Hong Kong RGC ECS (No. 24203419), RGC CRF (No. C4158-20G), CUHK Direct Grant (No. 4055114), and NSFC (No. U1936205). Zengfeng Huang is supported by Shanghai Sailing Program Grant No. 18YF1401200, Shanghai Science and Technology Commission Grant No. 17JC1420200.
\end{acks}

\bibliographystyle{ACM-Reference-Format}
\bibliography{rst1905-zhang}

\newpage

\appendix
\section{Proofs}
\noindent
{\bf Proof of Theorem \ref{thm-sppr-error-rate}}. We first prove the following lemma. 

\begin{lemma}\label{lem:sum1}
The probability that a supervised random walk stops at exactly the $k$-th hop is $\alpha_k\prod_{l=0}^{k-1}(1-\alpha_l)$ and for any k, we have
$$\alpha_k +\sum_{l=1}^{\infty}\alpha_{k+l}\prod_{j=0}^{k+l-1}(1-\alpha_j)=1.$$
\end{lemma}
\begin{proof}
Define $Prob_k(L)=\alpha_k+\sum_{l=1}^{L}\alpha_{k+l}\prod_{j=0}^{k+l-1}(1-\alpha_j)$.
We claim that $1-Prob_k(L)=\prod_{j=0}^{L}(1-\alpha_{k+j})$, implying that $(1-Prob_k(L))\to 0$ when $L \to \infty$ and thus, $Prob_k(L) \to 1$ when $L\to \infty$.
Now we justify $1 - Prob_k(L)=\prod_{j=0}^{L}(1 - \alpha_{k+j})$ by induction. When $L=0$, $1-Prob_k(0)=1-\alpha_k$ holds obviously. Suppose $1-Prob_k(t)=\prod_{j=0}^{t}(1-\alpha_{k+j})$ holds. Then for $Prob_k(t+1)=Prob_k(t)+\alpha_{k+t+1}(1-Prob_k(t))$, we can derive that:
\begin{displaymath}
    \begin{aligned}
        1-Prob_k(t+1)&= 1-Prob_k(t)-\alpha_{k+t+1}(1-Prob_k(t))\\
        &= (1-\alpha_{k+t+1})(1-Prob_k(t)) =\prod_{j=0}^{t+1}(1-\alpha_{k+j}).
    \end{aligned}
\end{displaymath}
Thus $1-Prob_k(L)=\prod_{j=0}^{L}(1-\alpha_{k+j})$ holds for any $L$. Proof done.
\end{proof}

\vspace{-2mm}
Note that the transition matrix $\vect{P}$ is a row-stochastic matrix, \textit{i.e.} $\vect{P}(i, j) \ge 0$ for any $i, j = 1, 2, \cdots ,n$, and $\sum_{j=1}^{n}\vect{P}(i, j)=1$ for any $i = 1, 2, \cdots ,n$. Following these properties, we can easily derive that $\vect{P}^l$ is also row-stochastic, where $l$ is the power of the matrix. Now we first look at the r.h.s. of the inequality in Theorem \ref{thm-sppr-error-rate}. Recall that the infinity-norm of matrix is the maximum absolute row sum, \textit{i.e.}, $||\vect{M}||_{\infty}=\max_{1 \le i \le n}\sum_{j=1}^n |\vect{M}(i,j)|$. Then, by Lemma \ref{lem:sum1}, for any $i=1, 2, \cdots, n$, we have
\begin{displaymath}
    \begin{aligned}
        \sum_{j=1}^n \vect{\prox}(i,j)&= \alpha_0 \sum_{j=1}^n \vect{P}^0(i,j) + \alpha_1 (1-\alpha_0) \sum_{j=1}^n \vect{P}^1(i,j) +  \cdots \\
        &= \alpha_0 + \sum_{l=1}^{\infty} \alpha_l \prod_{k=0}^{l-1} (1-\alpha_k) =1.
    \end{aligned}
\end{displaymath}
In terms of the l.h.s. of the inequality in Theorem \ref{thm-sppr-error-rate}, for any $i$-th row of matrix $\vect{\prox}-\vect{\prox}_L$, we have that: 
\begin{displaymath}
    \begin{aligned}
        & \sum_{j=1}^n (\vect{\prox}-\vect{\prox}_L)(i,j) \\
        =& \sum_{l=L+1}^{\infty}\alpha_{l} \prod_{k=0}^{l-1} (1-\alpha_k) \sum_{j=1}^n \vect{P}^{L+1}(i,j) = \sum_{l=L+1}^{\infty}\alpha_l\cdot \prod_{k=0}^{l-1}(1-\alpha_k)\\
        =& \prod_{k=0}^{L}(1-\alpha_k)(\alpha_{L+1}+\alpha_{L+2}(1-\alpha_{L+1})+\cdots) \le \prod_{k=0}^{L}(1-\alpha_k).
    \end{aligned}
\end{displaymath}
By combining the above derivations, 
$$\sum_{j=1}^n (\vect{\prox}-\vect{\prox}_L)(i,j) \le \prod_{k=0}^{L}(1-\alpha_k) \cdot \sum_{j=1}^n \vect{\prox}(i,j)$$ 
holds for any $i=1, 2, \cdots, n$. Thus, it clearly holds that 
$$||\vect{\prox}-\vect{\prox}_L||_{\infty} \le \prod_{k=0}^{L}(1-\alpha_k) ||\vect{\prox}||_{\infty}.$$
The theorem is proved.

\header
{\bf Proof of Theorem \ref{thm:gpush-time-complexity-1}}. The cost of Algorithm \ref{alg:alg-push} is dominated by the number of push operations. Recall that the push operation is invoked only if there exists an entry $\vect{r}^{(k)}_s(v)$ in residue vectors such that $\vect{r}^{(k)}_s(v) \ge \delta\cdot d_{out}(v)$. Let $\vect{\pi}_s^{(l)}(v)$ denote the $l$-hop supervised PPR whose value is the probability that the $l$-hop supervised random walk from $s$ stops at $v$. Then the total number of push operations caused by the residue value $\vect{r}^{(l)}_s(v)$ can be bounded by $\frac{\vect{\pi}_s^{(l)}(v)}{\alpha_l\delta \cdot d_{out}(v)}$. In addition, since in each push operation, the node $v$ will pass its residue to its out-neighbors, the cost for a push operation at $v$ is bounded by $O(d_{out}(v))$. Thus, the total cost of the push operations on entry $\vect{r}^{(l)}_s(v)$ is bounded by $\frac{\vect{\pi}_s^{(l)}(v)}{\alpha_l\delta \cdot d_{out}(v)} \cdot d_{out}(v)$. Based on the above analysis, the time cost of Algorithm \ref{alg:alg-push} is: 
\begin{displaymath}
    \sum_{l=0}^L \sum_{v\in V} \frac{\vect{\pi}_s^{(l)}(v)}{\alpha_l\delta \cdot d_{out}(v)} \cdot d_{out}(v) \le \frac{1}{\alpha_{min}\delta} \sum_{l=0}^L \sum_{v\in V}\vect{\pi}_s^{(l)}(v) \le \frac{1}{\alpha_{min}\delta},
\end{displaymath}
where $\alpha_{min} = \min\{\alpha_1,\cdots,\alpha_L\}$ can be treated as a constant. Thus, Algorithm \ref{alg:alg-push} runs in $O(\frac{1}{\alpha_{min}\delta})=O(\frac{1}{\delta})$ time and the proof is done.

\header
{\bf Proof of Theorem \ref{thm:gpush-time-complexity-2}}.
The following lemmas are used in the proof.

\begin{lemma}
{\rm \cite{computing}}
For undirected graph $G$ and any two nodes $u$ and $v$ in $G$, let $\vect{P}^k(u,v)$ (resp. $\vect{P}^k(v,u)$) be the probability of going from $u$ to $v$ (resp. from $v$ to $u$) through a random walk of fixed length $k$. Then,
$$
\frac{\vect{P}^k(u,v)}{d_{out}(v)} = \frac{\vect{P}^k(v,u)}{d_{out}(u)}.
$$
\end{lemma}

\begin{lemma}
\label{lem:rel-in-push}
Let $\vect{\pi}^{L}_s(u)$ denote the supervised PPR within $L$ hops. Then, after the end of every iteration in Algorithm \ref{alg:alg-push}, for $\vect{\hat{\pi}}_s(u)$ and the residue vectors $\vect{r}_s^{(0)}, \cdots \vect{r}_s^{(L)}$, we have
\begin{equation}
\label{equ:rel-in-push}
    \vect{\pi}_s^L(u)\le \vect{\hat{\pi}}_s(u)+\sum_{k=0}^{L} \sum_{v\in V} \vect{r}_s^{(k)}(v) \cdot h_u^k(v),
\end{equation}
where $h_u^k(v)=\alpha_k\vect{e}_u(v) + \sum_{l=1}^{\infty} \alpha_{k+l}\prod_{j=k}^{k+l-1}(1-\alpha_j)\vect{P}^l(v,u)$, and $\vect{e}_u$ is a unit vector in which the $u$-th entry of $\vect{e}_u$ is $1$ and others are all $0$.
\end{lemma}
\begin{proof}

We prove Lemma \ref{lem:rel-in-push} by induction. First, at the begin of Algorithm \ref{alg:alg-push}, vector $\vect{\hat{\pi}}_s$ and all residue vectors are set to $\vect{0}$ except for $\vect{r}_s^{(0)}(s)=1$. Then we have
\begin{displaymath}
    \begin{aligned}
        &\vect{\hat{\pi}}_s(u)+\sum_{k=0}^{K} \sum_{v\in V} \vect{r}_s^{(k)}(v) \cdot h_u^k(v)\\
        =&\, \alpha_0 + \sum_{l=1}^{\infty} \alpha_{l}\prod_{j=0}^{l-1}(1-\alpha_j)\vect{P}^l(s,u)=\vect{\pi}_s(u),
    \end{aligned}
\end{displaymath}
which implies that Equation \ref{equ:rel-in-push} holds under the initial condition since $\vect{\pi}_s^L(u) \le \vect{\pi}_s(u)$. Now suppose that Equation \ref{equ:rel-in-push} holds at the end of $i$-th iteration. Suppose further that during $(i+1)$-th iteration, entry $\vect{r}_s^{(w)}(q)\ge \delta\cdot d_{out}(q)$ is detected and push operation is invoked here. If the change of l.h.s of Equation \ref{equ:rel-in-push} after the push operation is zero, then Equation \ref{equ:rel-in-push} holds at the end of $(l+1)$-th iteration. 

For the convenience of analysis of the change, let $\vect{\hat{\pi}}'_s, \vect{\dot{r}}_s^{(w)}$ (resp. $\vect{\hat{\pi}}''_s, \vect{\ddot{r}}_s^{(w)}$) be the corresponding vectors at the end of $i$-th iteration(resp. $(i+1)$-th iteration). Then after performing the push operation in the $(i+1)$-th ieration, we have
$$\vect{\hat{\pi}}''_s(q)-\vect{\hat{\pi}}'_s(q)=\alpha_w\cdot\vect{\dot{r}}_s^{(w)}(q),$$
$$\vect{\ddot{r}}_s^{(w)}(q)-\vect{\dot{r}}_s^{(w)}(q)=-\vect{\dot{r}}_s^{(w)}(q).$$
Recap $N(q)$ is the set of $q$'s out-neighbors, for $o\in N(q)$, we have
$$\vect{\ddot{r}}_s^{(w+1)}(o)-\vect{\dot{r}}_s^{(w+1)}(o)=(1-\alpha_w)\frac{\vect{\dot{r}}_s^{(w)}(q)}{d_{out}(q)}.$$
Then the change between $i$-th iteration and $(i+1)$-th iteration is 
\begin{displaymath}
    \begin{aligned}
        &\vect{\hat{\pi}}''_s(u)-\vect{\hat{\pi}}'_s(u)+\vect{\ddot{r}}_s^{(w)}(q)\cdot h_u^w(q)-\vect{\dot{r}}_s^{(w)}(q)\cdot h_u^w(q)+\\
        &\sum_{o\in N(q)}\vect{\ddot{r}}_s^{(w+1)}(o)\cdot h_u^{w+1}(o)-\sum_{o\in N(q)}\vect{\dot{r}}_s^{(w+1)}(o)\cdot h_u^{w+1}(o)\\
    \end{aligned}
\end{displaymath}

\begin{displaymath}
    \begin{aligned}
        =&\,\vect{\hat{\pi}}''_s(u)-\vect{\hat{\pi}}'_s(u)-\vect{\dot{r}}_s^{(w)}(q)\cdot h_u^w(q)+\\
        &\,(1-\alpha_w)\frac{\vect{\dot{r}}_s^{(w)}(q)}{d_{out}(q)}\sum_{o\in N(q)}h_u^{w+1}(o)\\
        =&\,\vect{\hat{\pi}}''_s(u)-\vect{\hat{\pi}}'_s(u)-\vect{\dot{r}}_s^{(w)}(q)\cdot h_u^w(q)+\vect{\dot{r}}_s^{(w)}(q)(h_u^w(q)-\alpha_w\vect{e}_u(q))\\
        =&\,\vect{\hat{\pi}}''_s(u)-\vect{\hat{\pi}}'_s(u)-\alpha_w\vect{e}_u(q)\vect{\dot{r}}_s^{(w)}(q)=0.
    \end{aligned}
\end{displaymath}

Therefore, the change of l.h.s of Equation \ref{equ:rel-in-push} after $(i+1)$-th iteration is zero, which implies that Equation \ref{equ:rel-in-push} also holds at the end of $(i+1)$-th iteration. The lemma is proved.
\end{proof}

\noindent
Next, we prove Theorem \ref{thm:gpush-time-complexity-2}. By Equation \ref{equ:rel-in-push} in Lemma \ref{lem:rel-in-push}, we have:
\begin{displaymath}
    \begin{aligned}
        & |\vect{\pi}_s^L(u) - \vect{\hat{\pi}}_s(u)|/d_{out}(u)
        \le \sum_{k=0}^{L} \sum_{v\in V} \frac{\vect{r}_s^{(k)}(v)}{d_{out}(u)} \cdot h_u^k(v)\\
        \le & \, \delta \sum_{k=0}^{L} \sum_{v\in V}
        (\alpha_k\vect{e}_u(v) + \sum_{l=1}^{\infty} \alpha_{k+l}\prod_{j=k}^{k+l-1}(1-\alpha_j)\vect{P}^l(v,u))\\
        =& \, \delta\sum_{k=0}^L(\alpha_k + \sum_{l=1}^{\infty} \alpha_{k+l}\prod_{j=k}^{k+l-1}(1-\alpha_j))=L\cdot \delta,
    \end{aligned}
\end{displaymath}
where the last equation holds by Lemma \ref{lem:sum1}. The theorem is proved.

\header
{\bf Proof of Theorem \ref{thm:complexity-of-Lemane}}.
The time complexities of Algorithm \ref{alg:alg-gpt} depend on two parts: Generalized Push and SparseSVD. 
From the results in Theorem \ref{thm:gpush-time-complexity-1}, for each source $s\in V$, the cost of the push operations is $O(\frac{1}{\delta})$. Thus, the total cost of push operations on $n$ nodes is bounded by $O(\frac{n}{\delta})$. Then, we need the following theorem \cite{sparseembds} to analyze the cost of SparseSVD.
\begin{theorem}\label{thm:sparsesvd}
Let $\vect{A}$ denote an $n \times n$ matrix, there is an algorithm that, with failure probability $1/10$, finds two $n \times d$ matrices $\vect{U}, \vect{V}$ with orthonormal columns, and a $d \times d$ diagonal matrix $\vect{\Sigma}$, so that $\left ||\vect{A} - \vect{U}\vect{\Sigma}\vect{V}^T \right ||_F \leq (1+\epsilon) \left ||\vect{A}- \left[\vect{A}\right]_d \right ||_F$, where $[\vect{A}]_d$ denotes the best rank-$d$ approximation to $\vect{A}$. The algorithm runs in time 
$$O\left(\textrm{nnz}(\vect{A}) + \tilde{O}\left( nd^2 \epsilon^{-4} + d^3 \epsilon^{-5} \right) \right).$$
\end{theorem}

Racall that for any approximate supervised PPR score $\vect{\hat{\pi}}_u(v)$, it is add to $\vect{\prox}$ only if it is larger than the error parameter $\delta$. Thus the total number of non-zero entries in $\vect{\prox}$ is $\textrm{nnz}(\vect{\prox}) = O(\frac{n}{\delta})$. The cost of SparseSVD is bounded by $O(\frac{n}{\delta} + \frac{nd^2}{\epsilon^4})$. Finally, combining these two parts, the running cost of Algorithm \ref{alg:alg-gpt} is bounded by $O(\frac{n}{\delta} + \frac{nd^2}{\epsilon^4})$, which completes our proof.

\section{Hyper-parameters settings}

\begin{table}[t]
  \caption{Hyperparameters of {\lerand} for link prediction.}
  \vspace{-3mm}
  \label{tab:table-params1}
  \centering
  \begin{tabular}{l|l}
    \hline
    Dataset      & Hyperparameters \\
    \hline
    \multirow{3}{*}{Wikipedia}   & initialization: $\phi_{hk}(l)$ with $t=5$, \\
    ~            & learning rate: $0.001$, $\delta$: $10^{-5}$, $\beta$: $0.01$, $\gamma$: $1$,\\
    ~            & SVD used for push: JacobiSVD \\
    \hline
    \multirow{3}{*}{Wikivote}    & initialization: $\phi_{hk}(l)$ with $t=1$, \\
    ~            & learning rate: $0.5$, $\delta$: $10^{-6}$, $\beta$: $0.5$, $\gamma$: $1$,\\
    ~            & SVD used for push: frPCA    \\
    \hline
    \multirow{3}{*}{BlogCatalog} & initialization: $\phi_{ge}(l)$ with $\alpha=0.5$, \\
    ~            & learning rate: $0.1$, $\delta$: $10^{-7}$, $\beta$: $0.01$, $\gamma$: $1$,\\
    ~            & SVD used for push: frPCA    \\
    \hline
    \multirow{3}{*}{Slashdot}    & initialization: $\phi_{hk}(l)$ with $t=5$, \\
    ~            & learning rate: $0.001$, $\delta$: $10^{-5}$, $\beta$: $0.1$, $\gamma$: $1$,\\
    ~            & SVD used for push: frPCA    \\
    \hline
    \multirow{3}{*}{Tweibo}      & initialization: $\phi_{ge}(l)$ with $\alpha=0.5$, \\
    ~            & learning rate: $0.01$, $\delta$: $10^{-5}$, $\beta$: $0.1$, $\gamma$: $1$, \\
    ~            & SVD used for push: frPCA    \\
    \hline
    \multirow{3}{*}{Orkut}       & initialization: $\phi_{hk}(l)$ with $t=1$, \\
    ~            & learning rate: $0.01$, $\delta$: $10^{-4}$, $\beta$: $1$, $\gamma$: $1$, \\
    ~            & SVD used for push: frPCA     \\
    \hline
\end{tabular}
\end{table} 

\begin{table}[t]
  \caption{Hyperparameters of {\lerand} for node classification.}
    \vspace{-3mm}
  \label{tab:table-params2}
  \centering
  \begin{tabular}{l|l}
    \hline
    Dataset     & Hyperparameters\\
    \hline
    \multirow{3}{*}{Wikipedia}   & initialization: $\phi_{hk}(l)$ with $t=5$, \\
    ~            & learning rate: $0.05$, $\delta$: $10^{-5}$, $\beta'$: $1$, $\gamma'$: $0.5$, \\
    ~            & SVD used for push: JacobiSVD    \\
    \hline
    \multirow{3}{*}{BlogCatalog} & initialization: $\phi_{hk}(l)$ with $t=5$, \\
    ~            & learning rate: $0.01$, $\delta$: $10^{-5}$, $\beta'$: $1$, $\gamma'$: $0.5$, \\
    ~            & SVD used for push: JacobiSVD    \\
    \hline
    \multirow{5}{*}{TWeibo}      & initialization: $\phi_{hk}(l)$ with $t=5$, \\
    ~            & learning rate: $0.05$, $\delta$: $10^{-5}$, $\beta'$: $1$, $\gamma'$: $3$, \\
    ~            & SVD used for push: frPCA, $X$ and $Y$ are \\ ~ & normalized before concatenation in \\
    ~            & loss function $\mathcal{L}'_1$   \\
    \hline
    \multirow{5}{*}{Orkut}       & initialization: $\phi_{hk}(l)$ with $t=1$, \\
    ~            & learning rate: $0.5$ , $\delta$: $10^{-5}$, $\beta'$: $1$, $\gamma'$: $2$, \\
    ~            & SVD used for push: frPCA, $X$ and $Y$ are \\
    ~            & normalized before concatenation in  \\
    ~            & loss function $\mathcal{L}'_1$   \\
    \hline
\end{tabular}
\end{table}

Table \ref{tab:table-params1} and Table \ref{tab:table-params2} 
summarize the hyperparameter settings of {\lerand} on each dataset. The searching hyperparameters include initialized distribution $\phi(l)$ to indicate the probability that the supervised random walk stops at the $l$-th hop, the loss function coefficients $\beta,\gamma,\beta',\gamma'$, error parameter $\delta$, learning rate, and buildin SVDs used in generalized push. For the initialization of distribution $\phi(l)$, two standard distribution functions are used, namely the geometric distribution $\phi_{ge}(l) = \alpha (1-\alpha)^l$ and the Poisson distribution $\phi_{hk}(l) = \frac{e^{-t}t^l}{l!}$.

\end{document}